\documentclass[runningheads]{llncs}

\usepackage{latexsym,amsmath,amstext,amsfonts,amsbsy,amssymb,graphicx,hyperref}
\usepackage{algorithm,algorithmic,prettyref} \usepackage{leftidx}
\usepackage{times} 

\newcommand{\field}[1]{\mathbb{#1}} 
\newcommand{\Z}{\field{Z}} \newcommand{\E}{\field{E}}
\newcommand{\fancy}[1]{\mathcal{#1}} \newcommand{\M}{\fancy{M}}

\DeclareMathOperator*{\argmin}{arg\,min}

\newrefformat{def}{Definition~\ref{#1}}

\begin{document}

\title{Zeta Distribution and Transfer Learning Problem}
\author{Eray \"{O}zkural} 
\institute{Celestial Intellect Cybernetics \\celestialintellet.com}

\maketitle


\abstract{ We explore the relations between the zeta distribution and
  algorithmic information theory via a new model of the transfer
  learning problem. The program distribution is approximated by a zeta
  distribution with parameter near $1$.  We model the training
  sequence as a stochastic process.  We analyze the upper temporal
  bound for learning a training sequence and its entropy rates,
  assuming an oracle for the transfer learning problem. We argue from
  empirical evidence that power-law models are suitable for natural
  processes.  Four sequence models are proposed. Random typing model
  is like no-free lunch where transfer learning does not work. Zeta
  process independently samples programs from the zeta distribution.
  A model of common sub-programs inspired by genetics uses a database
  of sub-programs. An evolutionary zeta process samples mutations from
  Zeta distribution.  The analysis of stochastic processes inspired by
  evolution suggest that AI may be feasible in nature, countering
  no-free lunch sort of arguments.  }

\section{Introduction}
 
Although power-law distributions have been analyzed in depth in
physical sciences, little has been said about their relevance to
Artificial Intelligence (AI). We introduce the zeta distribution as an
analytic device in algorithmic information theory and propose using it
to approximate the distribution of programs.  We have been inspired by
the empirical evidence in complex systems, especially biology and
genetics, that show an abundance of power-law distributions in
nature. It is well possible that the famous universal distribution in
AI theory is closely related to power-law distributions in complex
systems.

The transfer learning problem also merits our attention, as a general
model of it has not been presented in machine learning literature.  We
develop a basic formalization of the problem using stochastic
processes and introduce temporal bounds for learning a training
sequence of induction problems, and transfer learning. The entropy
rate of a stochastic process emerges as a critical quantity in these
bounds. We show how to apply the bounds by analyzing the entropy rates
of simple training sequence models that generate programs. Two models
are close to what critics of AI have imagined, and easily result in
unsolvable problems, while two models inspired by evolution suggest
that there may be stochastic processes in nature on which AGI
algorithms may be quite effective.

\section{Approximating the Distribution of Programs}

Solomonoff's universal distribution depends on the probability
distribution of programs. A natural model is to consider programs, the
bits of which are generated by a fair coin. Solomonoff defined the
probability of a program $\pi \in \{0,1\}^+$ as:
\begin{equation}
  \label{eq:prog-dist}
  P(\pi) = 2^{-|\pi|}
\end{equation}
where $|\pi|$ is the program length in bits. The total probability of
all programs thus defined unfortunately diverges if all bit-strings
$\pi \in \{0,1\}^*$ are considered valid programs. For constructing
probability distributions, a convergent sum is required. Extended
Kraft inequality shows that the total probability is less than $1$ for
a prefix-free set of infinite programs \cite{Cover1991}.  Let $M$ be a
reference machine which runs programs with a prefix-free encoding like LISP.
The algorithmic probability that a bit-string $x \in \{0,1\}^*$ is
generated by a random program of $M$ is:
\begin{equation}
  \label{eq:alp}
  P_M(x) = \sum_{M(\pi) = x*} P(\pi)  
\end{equation}
which conforms to Kolmogorov's axioms \cite{levin-thesis}.  $P_M$ is
also called the universal prior for it may be used as the prior in
Bayesian inference, as any data can be encoded as a bit-string.

\subsection{Zeta Distribution of Programs}

We propose the zeta distribution for approximating the distribution of
programs of $M$.  The distribution of \prettyref{eq:prog-dist} is
already an approximation, even after normalization
, since it contains many programs that are semantically incorrect, and
those that do not generate any strings. A realistic program
distribution requires us to specify a detailed probability model of
programs, which is not covered by the general model, however, the
general model, which is approximate, still gives excellent bounds on
the limits of Solomonoff's universal induction method. Therefore,
other general approximations may also be considered.

Additionally, the zeta function is universal, which encourages us to
relate algorithmic information theory to zeta distribution
\cite{Voronin75}.

Let us consider a program bit-string
  $\pi = b_1b_2b_3\dots b_k$.
Let $\phi: \{0,1\}^+ \rightarrow \Z$ define the arithmetization of
programs represented as bit-strings, where the first bit is the most
significant bit.
\begin{equation}
  \label{eq:arithmetization}
  \phi(\pi) = \sum_{i=1}^{i \leq |\pi|} b_i.2^{ |\pi|-i}
\end{equation}
Thus arithmetized, we now show a simple, but interesting 
inequality about the distribution of programs:
\begin{align}
  P(\pi) &= 2^{-\lceil \log_2(\phi(\pi)+1) \rceil} \\
  (2a)^{-1} &\leq 2^{-\lceil \log_2 a\rceil} \leq a^{-1}, \text{for } a\geq4\\
  \label{eq:sandwich}
  (2(\phi(\pi)+1))^{-1} &\leq P(\pi) \leq (\phi(\pi)+1)^{-1}, \text{for } \phi(\pi)\geq 3
\end{align} 
which shows an approximation that is closer than a factor of $2$. Program codes $\phi(\pi) < 3$ are discarded. 

Zipf's law $f_n \alpha n^{-1}$ manifests itself as the Zipf
distribution of ranked discrete objects $\{ o_1,o_2,\dots,o_n \}$ in
order of increasing rank $i$
\begin{equation}
  \label{eq:zipf}
  P( Z_s^{(n)} = o_i) \triangleq \frac{1}{i^sZ}
\end{equation}
where $Z_s^{(n)}$ is a random variable, $Z$ is the normalization
constant and $s \geq 1$ (we used the notation $Z_s^{(n)}$ simply to
avoid confusion with exponentiation, $Z_s$ is a standard notation for
the zeta random variable).  Zeta distribution is the countably
infinite version of Zipf distribution with parameter $s>1$
\begin{equation}
  \label{eq:zeta}
  P( Z_s = k) = \frac{1}{k^s.\zeta(s)}
\end{equation}
where $Z_s$ is a random variable with co-domain $\Z^+$ and the zeta
function is defined as
\begin{equation}
  \label{eq:2}
  \zeta(s) = \sum_{n=1}^{\infty}\frac{1}{n^s} \quad.
\end{equation}
Note that Zeta distribution is a discrete variant of Pareto
distribution.


It is much involved to work with a prefix-free set, therefore we will
suggest an alternative device to approximate $P(\pi)$.
\begin{theorem}
  \label{thm:zipf-approx}
  A program distribution may be approximated by the Zipf distribution
  with $s=1$, or by the zeta distribution with a real $s$ close to $1$
  from above.
\end{theorem}
\begin{proof}

  (a) Zeta distribution is undefined for $s=1$. However, if we use the
  Zipf distribution instead, and model programs up to a fixed
  program-length, we can approximate the program distribution from
  above using $(\phi(\pi)+1)^{-1}$ and from below using
  $ (2\phi(\pi)+2)^{-1}$ due to the sandwich property
  \prettyref{eq:sandwich}.

  (b) We can approximate the program distribution from below using
  $(2\phi(\pi)+2)^{-1}$. Since
  \begin{equation*}
    \forall \epsilon >0,  \ (2\phi(\pi)+2)^{-(1+\epsilon)} \leq
    (2\phi(\pi)+2)^{-1} < P(\pi) ,
  \end{equation*}
  we can also approximate it with the Zeta distribution
  \prettyref{eq:zeta} for $s$ close to $1$.
\end{proof}
In either case, the need for a prefix-free set of programs is
obviated. Of the simplified distribution, we investigate if the
approximations are usable.
\begin{theorem}
  \label{thm:zipf-conv}
  The program distribution $P(\pi)$ asymptotically obeys a power law
  with exponent $-1$ as program size grows.
\end{theorem}
\begin{proof}
  The probability of arithmetized program $\pi$ is sandwiched between
  $(\phi(\pi)+1)^{-1}$ and $(2\phi(\pi)+2)^{-1}$, therefore as
  $|\pi|$ grows, Zipf's law grows closer to $P(\pi)$.
  \begin{equation}
    \label{eq:zipf-conv1}
    \lim_{|\pi| \to \infty}   (\phi(\pi)+1)^{-1} - (2\phi(\pi)+2)^{-1} = 0 
  \end{equation} 
  \begin{equation}
    \label{eq:zipf-conv2}
    \lim_{|\pi| \to \infty}   2^{-|\pi|} - (2\phi(\pi)+2)^{-1}  =
    \lim_{|\pi| \to \infty}   (\phi(\pi)+1)^{-1} - 2^{-|\pi|} = 0 
  \end{equation}
\end{proof}

Combining \prettyref{thm:zipf-approx} and \prettyref{thm:zipf-conv},
we propose using a Zeta distribution with a parameter close to $1$. Obviously, lower and upper bounds vary only by a factor of $2$ within each other, therefore the error in the approximation of program distribution is at most by $1$ bit (this property will be analyzed in detail in an extended version of the present paper).
Substituting into \prettyref{eq:alp}, we propose an approximation
\begin{definition}
  \label{def:alp-zeta}
  \begin{equation}
    \label{eq:alp-zeta}
    P_M(x) \approxeq \sum_{M(\pi) = x*} \frac{1}{(\phi(\pi)+1)^{1+\epsilon}.\zeta(1+\epsilon)}
  \end{equation}
\end{definition}
where $\zeta(1+\epsilon) \geq 2$ ($\zeta(1.7) \approxeq 2$). \prettyref{def:alp-zeta}
 may be useful for machine learning theorists wherever
they must represent a priori program probabilities, as it allows them
to employ number theory. See Elias Gamma Code \cite{elias} for an alternative
integer code. 

\section{Training Sequence as a Stochastic Process}

Although Solomonoff has theoretically described how the transfer
learning problem might be solved in \cite{solomonoff-incremental}, a
detailed theoretical model of transfer learning for the universal
induction setting is missing in the literature. Here, we attempt to
fill this gap. In his treatise of incremental learning, Solomonoff
approached the transfer learning problem by describing an update
problem which improves the guiding conditional probability
distribution (GCPD) of the system as an inductive inference problem of
the type that the system usually solves. Solomonoff's modular approach
started with a number of problem solving methods and invented new such
methods as the system progressed. The initial methods, however, are
not fully specified, and we leave it as an open problem in this
paper. Instead, we attempt at describing the space of training
sequences using the zeta distribution, showing an interesting
similarity to our world, whereas most problems in a sequence may be
solved, but rarely they are not solvable at all. For instance, a
mathematician may solve most problems, but stall at a conjecture that
requires the invention of a new, non-trivial axiom indefinitely.

In usual Solomonoff induction (with no transfer learning component), a
computable stochastic source $\mu$ is assumed. The stochastic source
may generate sequences, sets, functions, or other structures that we
please, the general law of which may be induced via Solomonoff's
method. We extend Solomonoff's induction model to a training sequence
of induction problems, by considering a stochastic process $\M$ of $n$
random variables.
\begin{equation}
  \label{eq:training-sequence}
  \M = \{ \mu_1, \mu_2, \mu_3, \dots, \mu_n \}
\end{equation}

The transfer learning problem thus is constituted from solving $n$
induction problems in sequence which are generated from the stochastic
process $\M$. It does not matter which type of induction problem these
problems are, as long as they are generated via $\M$.
 
\subsection{Entropy Rate of a Training Sequence}

A critical measurement of a stochastic process is its entropy rate,
which is defined as the following for $\M$:
\begin{equation}
  \label{eq:entropy-rate}
  H(\M) = \lim_{n \to \infty} \frac{H(\mu_1,\mu_2,\mu_3, \dots, \mu_n)}{n}
\end{equation}
and the conditional entropy rate,
\begin{equation}
  \label{eq:cond-entropy-rate}
  H'(\M) = \lim_{n \to \infty} \frac{H( \mu_n | \mu_1,\mu_2,\mu_3, \dots, \mu_{n-1})}{n}
\end{equation}
which gives the entropy given past observations.
Observe that there is a well-known relation between average Kolmogorov
complexity and the entropy of an i.i.d. stochastic process (Equation 5
in \cite{universality-zipf}):
\begin{equation}
  \label{eq:kolmogorov-shannon}
  \lim_{n \to \infty} \frac{K_M(X_1,X_2,X_3, \dots, X_n)}{n} =
  H(X) + O(1)
\end{equation}
where $X$ is a stochastic process and $X_i$ its random variables. 
We assume that the relation extends to conditional entropy without proof
due to lack of space.

\subsection{Training Time}

Let $\pi^*_i$ be the minimal program for exactly simulating $\mu_i$ on
$M$. The most general expression for $\pi^*_i$ is given in the
following
\begin{equation}
  \label{eq:minimal}
  \pi^*_i = \argmin_{\pi_j}(\{ |\pi_j| \ | \ \forall x,y \in \{0,1\}^*:  M(\pi_j,x,y)=P(\mu_i =x | y) \})
\end{equation}
where the pdf of stochastic source $\mu_i$ is simulated by a program
$\pi_j$. The conditional parameter $y$ is optional. Let us note the
following identity
\begin{equation}
  \label{eq:sim-length}
  K_M(\mu_i) = |\pi^*_i|
\end{equation}
since arguments $x,y$ are extraneous input to the pdf specified by
$\pi^*_i$.

Let $t(\mu_i)$ denote the time taken to solve $\mu_i$, and $t(\pi)$
denote the time taken by program $\pi$ on M. Assume that 
$t(\mu_i) < \infty$. 
We know that the running time
of extended Levin Search is bias-optimal
\cite{solomonoff-incremental}, and  
\begin{equation}
  \label{eq:cjs}
  \frac{t(\pi^*_i)} { P(\pi^*_i)} \leq t(\mu_i) \leq \frac{2 t(\pi^*_i)} { P(\pi^*_i)} 
\end{equation}
for a computable stochastic source $\mu_i$ ($K_M(\mu_i)<\infty$). 
 The lower bound in
\prettyref{eq:cjs} has been named conceptual jump size by Solomonoff,
because it refers to the solution of individual induction problems within
a training sequence, quantifying how much conceptual innovation is
required for a new problem \cite{solomonoff-incremental}.  We cannot
exactly predict $t(\mu_i)$ due to the incomputability of algorithmic
probability.  Extended Levin Search will keep running indefinitely. It
is up to the user to stop execution, which is usually bounded only by
the amount of computational resources available to the user.
We should also mention that Levin himself does not think
that any realistic problems can be solved by Levin search or created
on a computer \cite{levin-forbidden}. In the present paper, we run
counter to Levin's position, by arguing that Levin search can work in
an evolutionary setting, assuming an $O(1)$ oracle for the transfer
learning problem.

We substitute the relation between $K_M(x)$ and $P_M(x)$ in the upper
bound for $t(\mu_i)$,
\begin{align}
  \label{eq:time}
  K_M(\pi^*_i) &= -\log_2{P(\pi^*_i)}
\end{align}
obtaining the following fact due to \prettyref{eq:sim-length} and \prettyref{eq:time}:
\begin{lemma}
  \label{lem:time2}
  $ t(\mu_i) \leq 2t(\pi^*_i) 2^{K_M(\mu_i)}$
\end{lemma}
The inequality translates to the time for the training sequence $\M$
as
\begin{theorem}
  \label{thm:process-time}
  \begin{equation}
    \label{eq:5}
    t(\M)  \leq \sum_{i=1}^n t(\pi_i^*) 2^{K_M(\mu_i)+1} 
  \end{equation}
\end{theorem}
which is a simple sum of \prettyref{lem:time2}. 

The conditional entropy rate is useful when the stochastic process has
inter-dependence. Let us define conditional Kolmogorov complexity for the
training sequence $\M$,
\begin{equation}
  \label{eq:conditional-entropy}
  K'(\M_{<k}) \triangleq K( \mu_k | \mu_1,\mu_2,\mu_3, \dots, \mu_{k-1})
\end{equation}
where $\M_{<k} \triangleq \{ \mu_i | i \leq k  \} $. We define likewise for the stochastic process probabilities. 
\begin{equation}
  \label{eq:flow}
  P'(\M_{<k}) \triangleq P( \mu_k | \mu_1,\mu_2,\mu_3, \dots, \mu_{k-1})
\end{equation}
$K'(\M_{<k})$ captures new algorithmic information content for the $k^{th}$
variable of the stochastic process given the entire history.

As $n$ grows, the transfer learning oracle has to add
$H'(\M)$ bits of information to its memory on the average in the
stochastic process $\M$ as Kolmogorov-Shannon entropy relation
\prettyref{eq:kolmogorov-shannon} holds in the limit for conditional
entropy, as well. Since the upper temporal bound grows exponentially,
\prettyref{eq:conditional-entropy} only relates loosely to the
solution time $t(\mu_i)$ of a particular problem.  We instead define
the conditional expected training time upper bound with respect to $\M$:
\begin{equation}
  \label{eq:condexpectedtime}
  \E'[t(\M_{<k})] \triangleq \E_{\M}[t(\mu_k)  | \mu_1, \dots, \mu_{k-1}] \leq \sum_{\forall \mu_k \in \{0,1\}*}2t(\pi_k^*) 2^{K'(\M_{<k})} P'(\M_{<k}) 
\end{equation}

\subsection{Random Typing Model}

Let us start by considering the well-known model of random typing.  If
each $\mu_i$ is regarded as a random $m$-bit program out of $2^m$ such
programs, the programs are independent, and the entropy rate is $m$
bits exactly (under usual i.i.d. assumptions, e.g., we are using fair
coin tosses, and we construct programs using a binary
alphabet). Assume $2^m >> n$.

In the random typing model, all $\mu_i$ are algorithmically
independent, therefore there is no saving that can be achieved by
transfer learning. The time it takes for any problem is therefore:
\begin{align}
  \label{eq:3}
  t(\mu_i) &\leq t(\pi_i^*) 2^{m+1}
\end{align}
for any of the $2^m$ programs. Since $m$ can be arbitrarily large,
this model is compatible with Levin's conjecture that AI is
impossible. Note that this simplistic model is reminiscient of various
no-free lunch theorems that were heralded as mathematical proof that
general-purpose machine learning was impossible. However, this
scenario is highly unrealistic. It is extremely difficult to find
problems that are completely independent, as this would require us to
be using true random number generators to generate any problem. In
other words, we are only showing this ``model'' to demonstrate how far
removed from reality no-free lunch theorems are. In a physical world,
this model would correspond to the claim that quantum randomness
saturates every observation we may make. However, we already know this
claim to be false, since our observations do not consist of noise.  On
the contrary, there is a lot of dependable regularity in the
environment we inhabit, which is sometimes termed ``commmon sense'' in
AI literature.

\subsection{Power-law in Nature}

A more realistic model, however, uses the zeta distribution for
programs instead of uniform distribution. We propose this indeed to be
the case since zeta distribution is empirically observed in a
multitude of domains, and has good theoretical justification for the
abundance of power-law in nature.  \prettyref{thm:zipf-conv} gives
some weak and indirect justification as to why we might observe
fractions of the zeta distribution of programs in a computable
universe. However, there are more direct and appealing reasons why we
must expect to see the zeta distribution in highly evolved complex
systems. First, it is a direct consequence of the power-law ansatz,
and scale-invariance \cite{universality-zipf} or preferential
attachment in evolutionary systems \cite{yule}.  Second, it follows
from an application of maximum entropy principle where the mean of
logarithms of observations is fixed \cite{visser-zipf}.  Third,
biologists have observed the zeta distribution directly in genetic
evolution, thus strengthening the case that our $\pi^*_i$'s are likely
to conform to zeta distributions.  For instance, gene family sizes
versus their frequencies follow a power-law distribution
\cite{huynen-freqdist} and the gene expression in various species
follows Zipf's law \cite{furusawa-zipf}. Universal regularities in
evolution have been observed, for instance in the power-law relation
between the number of gene families and gene family size, and number
of genes in a category versus number of genes in genome, and power-law
like distribution of network node degree
\cite{koonin-laws-evolution}. Therefore, there is not only a highly
theoretical heuristic argument that we are following, but there exist
multiple theoretical and empirical justifications for expecting to
observe the zeta distribution of programs in nature. The material
evolution of the environment in a habitat, is not altogether different
from biological evolution. Except in the case of rare natural
catastrophes, the material environment changes only gradually in
accord with the dynamic flow of natural law (surprise is small), and
is dependent mostly on the actions of organisms in a complex habitat,
which may be considered to be programs from an information-theoretic
point of view.  In that sense, the entire ecology of the habitat in
question may be considered to be an evolutionary system, with program
frequencies similar to the case of genes in a single organism.  In the
following, we introduce novel models of training sequences inspired
by these empirical justifications.

\subsection{Identical Zeta Random Variables}

Let $\M$ be i.i.d. generated from zeta distribution according to
\prettyref{thm:zipf-conv}. Then,
\begin{equation}
  \label{eq:zeta-iid}
  H'(\M) = H(\mu_1) = H(Z_s) 
\end{equation}
indicating that the constant entropy rate depends only on the entropy of 
the zeta distribution.  
We thus analyze the running time. Let $t_{max}=\max\ \{ t(\mu_i) \}$.
\begin{equation}
  \label{eq:zeta-iid}
\E'[t(\M_{<k})] \leq 
{2  t_{max} \over \zeta(s)} \sum_{k=1}^\infty 2^{\lceil \log_2 k \rceil} k^{-s} \leq
{4 t_{max} \over \zeta(s)} \sum_{k=1}^\infty {k \over k^{s}}
\end{equation}
For the first 1 trillion programs, 
$t_{max}\sum_{k=1}^{10^{12}} 4k / k^{1.001}\zeta(1.001) \approxeq 3.89 \times 10^{9} t_{max}$ for $s=1.001$, which is a feasible factor for a realistic program search limit.


Note that AI theorists interpret i.i.d. assumptions as the main reason
why no free-lunch theorems are unrealistic \cite{lattimore2013}. Our
i.i.d. zeta process here may be interpreted as an elaboration of that
particular objection to no free-lunch theorems. Therefore,
we follow the heuristic argument that the right description of the
environment which we observe must be something else than the
random typing model since agents succeed in transfer learning. The
constant zeta process leans towards feasibility, but it does not yet model 
transfer learning in complex environments.


\subsection{Zipf Distribution of Sub-programs}

Based upon the observations of genetic evolution above and the fact
that the whole ecology is an evolutionary system, we may consider a
process of programs that has the following property. Each $\pi^*_i$
that corresponds to $\mu_i$ is constructed from a number of
sub-programs (concatenated). The joint distribution of sub-programs is
$Z_s^{(n)}$. This is a model of gene frequencies observed in
chromosomes, where each chromosome corresponds to a program, and each
gene corresponds to a sub-program.  Such a distribution would more
closely model a realistic distribution of programs by constraining
possible programs, as in the real-world the process that generates
programs is not ergodic. The total entropy of the process therefore
depends on the sub-programs that may be assumed to be random, and
program coding. Let each sub-program be a $k$-bit random program for
the sake of simplicity.  The sub-programs that correspond to
instructions are specified in a database of $2^k$ bits.
Instructions are not equiprobable, however, as in the random typing
model. Let each program have $m$ instructions drawn from the set of
$2^k$ instructions:
\begin{equation}
  A = \{ a_1, a_2, a_3, \dots, a_{2^k}\} . \label{eq:5}
\end{equation}
Then, we can model each optimal program $\pi^*_i$ as
\begin{equation}
  \label{eq:7}
  \pi^*_i = \pi^*_{i,1} \pi^*_{i,2} \pi^*_{i,3}\dots \pi^*_{i,m} 
\end{equation}
which make up a matrix of instructions $P^* = \pi^*_{i,j}$ where
$\pi^*_{i,j}$ is drawn from the set $A$ of instructions.  The total
entropy is due to the database of sub-programs, and the entropy of the
global distribution of sub-programs $Z_s^{(n)}$ which determines the
entropy of $P^*$.  The total entropy is then approximately,
\begin{equation}
  \label{eq:total-entropy}
  H(\mu_1,\mu_2,\dots,\mu_n) \approx 
  \log_2k + k.2^k + \log_2n + \log_2m+ H(Z_s^{(2^k)}) 
\end{equation}
where we show the significant terms for $k,n,m,$ parameters.
\begin{lemma} For the Zipf distribution of sub-programs,
  \begin{equation}
    \label{eq:zipf-ent-rate}
    H'(\M) \approx \lim_{n \to \infty} \frac{1}{n} \Big( k.2^k +  
    \frac{s}{H_{2^k,s}}\sum_{l=1}^{2^k}\frac{\ln(l)}{l^s} + \ln(H_{2^k,s})
    \\ + \log_2k + \log_2n + \log_2m
    \Big)
  \end{equation}
  due to \prettyref{eq:total-entropy}.
\end{lemma}
which is to say that, the entropy rate, and thus running time,
critically depends on the choice of $k$ and $n$.

\subsection{An Evolutionary Zeta Process}

Another process of programs may be determined by mimicking evolution,
by considering random mutations of programs in a training
sequence. Let us set
\begin{align}
  \label{eq:9}
  \pi^*_1 &= \wedge \\
  \label{eq:mutation}
  \pi^*_i & = 
            \begin{cases} 
              M(Z_s,\pi^*_{i-1}), &\text{if } Z_s \text{ is a valid transformation} \\
              \pi^*_{i-1}, & \text{otherwise}
            \end{cases}
\end{align}
which would apply a random transformation sampled from $Z_s$ in
sequence to an initially null program. Such mutations are unlikely to be too
complex. The resulting process has small conditional entropy rate,
which is wholly dependent on $Z_s$.
\begin{equation}
  \label{eq:10}
  \lim_{n \to \infty} H'(\M) = H(Z_s) = \log(\zeta(s)) - \frac{s\zeta'(s)}{\zeta(s)} 
\end{equation}
\begin{lemma}
 \begin{align}
   \label{eq:h-zeta-vals}
   H(Z_{1.1}) & = 13.8    & H(Z_{1.05}) & = 24.5   \\   
   \label{eq:h-zeta-vals2}
   H(Z_{1.01}) & = 106.1 & H(Z_{1.001}) & = 1008.4      
 \end{align}
\end{lemma}
The lemma suggests that if an evolutionary process evolves slowly
enough, then an AI can easily learn everything there is to learn about
it provided that the time complexity of random variables is not too
large. We can also employ $Z_s^{(k)}$ instead
of $Z_s$ in \prettyref{eq:mutation}. For a universal induction approximation,
$Z_{1.001}$ may be difficult to handle, however, for efficient model-based 
learning algorithms such as gradient descent methods, digesting new information 
on the order of a thousand bits is not a big challenge given sufficiently many 
samples for a problem $\mu_i$ in the sequence.

\section{Concluding Remarks}

We have shown novel relations between Zipf's law and program
distribution by means of the arithmetization of programs. We have
shown that zeta distribution may be used for approximating program
distributions. We have proposed using the conditional entropy rate as
an informative quantity for transfer learning.  We have extended
Solomonoff's induction model to a training sequence of problems as a
stochastic process. We have proposed that the entropy rate of a
stochastic process is informative.  We have defined conditional
Kolmogorov complexity and probability for the sequence, and have used
these quantities to define a conditional expected upper bound of
training time assuming an $O(1)$ transfer learning oracle.  We
introduced sequence models to show that there is a wide range of
possible stochastic processes that may be used to argue for the
possibility of general purpose AI. The random typing model is a
sensible elaboration of no-free lunch theorem kind of arguments, and
demonstrate how artificial and unlikely they are since everything is
interconnected in nature and pure randomness is very hard to come by,
which we therefore rule out as a plausible model of transfer learning.
We have shown several empirical justifications for using a power-law 
model of natural processes. 
Independent Zeta process tends to the feasible, but does not explain
transfer learning. The models that were inspired by natural evolution
allow general purpose learning to be feasible. 
In particular, the model of common sub-programs which is inspired
by empirical evidence in genetics supports a view of evolution of
natural processes that allows incremental learning to be
effective. The evolutionary Zeta process applies random mutations,
which can be slow enough for a machine learning algorithm to digest all
the new information.

A more detailed analysis of the transfer learning problem will be
presented in an extended journal paper. Open problems include analyzing
the complexity of the optimal update algorithm, time complexity analysis 
for the evolutionary processes, and accounting for the time complexity of 
individual programs.


\section*{Acknowledgements}

The paper was substantially improved owing to the extensive 
and helpful comments of anonymous AGI 2014 and AGI 2018 reviewers.

\bibliographystyle{splncs03} \bibliography{agi,physics}

\begin{thebibliography}{10}
\providecommand{\url}[1]{\texttt{#1}}
\providecommand{\urlprefix}{URL }

\bibitem{universality-zipf}
Corominas-Murtra, B., Sol\'e, R.V.: Universality of zipf’s law. Phys. Rev. E
  82,  011102 (Jul 2010),
  \url{http://link.aps.org/doi/10.1103/PhysRevE.82.011102}

\bibitem{Cover1991}
Cover, T.M., Thomas, J.A.: Elements of Information Theory. Wiley-Interscience,
  New York, NY, USA (1991)

\bibitem{elias}
Elias, P.: Universal codeword sets and representations of the integers. IEEE
  Transactions on Information Theory  (1975)

\bibitem{furusawa-zipf}
Furusawa, C., Kaneko, K.: Zipf’s law in gene expression. Phys. Rev. Lett.
  90,  088102 (Feb 2003),
  \url{http://link.aps.org/doi/10.1103/PhysRevLett.90.088102}

\bibitem{huynen-freqdist}
Huynen, M.A., van Nimwegen, E.: The frequency distribution of gene family sizes
  in complete genomes. Molecular Biology and Evolution  15(5),  583--589
  (1998), \url{http://mbe.oxfordjournals.org/content/15/5/583.abstract}

\bibitem{koonin-laws-evolution}
Koonin, E.V.: Are there laws of genome evolution? PLoS Comput Biol  7(8),
  e1002173 (08 2011), \url{http://dx.doi.org/10.1371%2Fjournal.pcbi.1002173}

\bibitem{lattimore2013}
Lattimore, T., Hutter, M.: No free lunch versus occam’s razor in supervised
  learning. In: Dowe, D. (ed.) Algorithmic Probability and Friends. Bayesian
  Prediction and Artificial Intelligence, Lecture Notes in Computer Science,
  vol. 7070, pp. 223--235. Springer Berlin Heidelberg (2013),
  \url{http://dx.doi.org/10.1007/978-3-642-44958-1_17}

\bibitem{levin-forbidden}
{Levin}, L.A.: {Forbidden Information}. eprint arXiv:cs/0203029  (Mar 2002)

\bibitem{levin-thesis}
Levin, L.A.: Some theorems on the algorithmic approach to probability theory
  and information theory. CoRR  abs/1009.5894 (2010)

\bibitem{solomonoff-incremental}
Solomonoff, R.J.: A system for incremental learning based on algorithmic
  probability. In: Proceedings of the Sixth Israeli Conference on Artificial
  Intelligence. pp. 515--527. Tel Aviv, Israel (December 1989)

\bibitem{visser-zipf}
Visser, M.: Zipf's law, power laws and maximum entropy. New Journal of Physics
  15(4),  043021 (2013), \url{http://stacks.iop.org/1367-2630/15/i=4/a=043021}

\bibitem{Voronin75}
Voronin, S.M.: Theorem on the ``universality'' of the riemann zeta-function.
  Izv. Akad. Nauk SSSR Ser. Mat.  39(3),  475--486 (1975)

\bibitem{yule}
Yule, G.U.: A mathematical theory of evolution, based on the conclusions of dr.
  j. c. willis, f.r.s. Philosophical Transactions of the Royal Society of
  London. Series B, Containing Papers of a Biological Character  213(402-410),
  21--87 (1925),
  \url{http://rstb.royalsocietypublishing.org/content/213/402-410/21.short}

\end{thebibliography}

\end{document}